\documentclass{article}

\usepackage[justification=centering]{caption} 
\usepackage{microtype}
\usepackage{graphicx}
\usepackage{subfigure}
\usepackage{booktabs} 
\usepackage{url}
\usepackage{amssymb,amsmath}
\usepackage{hyperref}

\usepackage{amsthm}
\newtheorem{prop}{Proposition}

\usepackage[accepted]{icml2018}

\icmltitlerunning{Joint-stochastic-approximation Random Fields with Application to Semi-supervised Learning}

\begin{document}

\twocolumn[
\icmltitle{Joint-stochastic-approximation Random Fields with Application to Semi-supervised Learning}
\icmlsetsymbol{equal}{*}

\begin{icmlauthorlist}
\icmlauthor{Yunfu~Song}{thuee}
\icmlauthor{Zhijian~Ou}{thuee}
\end{icmlauthorlist}

\icmlaffiliation{thuee}{Department of Electronic Engineering, Tsinghua University, Beijing, China}
\icmlcorrespondingauthor{Zhijian Ou}{ozj@tsinghua.edu.cn}

\vskip 0.3in
]


\printAffiliationsAndNotice{} 

\begin{abstract}
 
Our examination of deep generative models (DGMs) developed for semi-supervised learning (SSL), mainly GANs and VAEs, reveals two problems.
First, mode missing and mode covering phenomenons are observed in genertion with GANs and VAEs.
Second, there exists an awkward conflict between good classification and good generation in SSL by employing directed generative models.
To address these problems, we formally present joint-stochastic-approximation random fields (JRFs) --
a new family of algorithms for building deep undirected generative models, with application to SSL.
It is found through synthetic experiments that JRFs work well in balancing mode covering and mode missing, and match the empirical data distribution well.
Empirically, JRFs achieve good classification results comparable to the state-of-art methods on widely adopted datasets -- MNIST, SVHN, and CIFAR-10 in SSL,
and simultaneously perform good generation.

\end{abstract}

\section{Introduction}

Semi-supervised learning (SSL) considers the problem of classification when only a small subset of the observations have corresponding class labels, and aims to leverage the large amount of unlabeled data to boost the classification performance.
Several broad classes of methods for semi-supervised learning include generative models \cite{zhu05}, transductive SVM \cite{svm}, co-training \cite{co-training}, and graph-based methods (see \cite{zhu05} for more introduction).
In recent years, significant progress has been made on  representation, learning and inference with Deep Generative Models (DGMs) \cite{hm,dbn,dbm,vae,gan,gmmn,jsa}, and this stimulates an explosion of interest in utilizing DGMs for semi-supervised learning.

Semi-supervised learning with DGMs usually involves blending unsupervised learning and supervised learning. 
One justification is that the unsupervised loss (e.g. the negative marginal likelihood over the unlabeled data) provides additional regularization for the supervised loss over the labeled training data [Why Does Unsupervised Pre-training Help Deep Learning?]\cite{zhu05}.
Therefore, successful SSL methods often develop or adapt from unsupervised learning methods for DGMs.

Generally speaking, there are two broad classes of DGMs. 
Deep directed generative models, have been greatly advanced by inventing new models with new learning algorithms, such as Helmholtz Machines \cite{hm}, Variational Autoencoders (VAEs) \cite{vae}, Generative Adversarial Networks (GANs) \cite{gan}, auto-regressive neural networks \cite{auto-re} and so on.
Currently, the two most prominent DGM techniques - VAEs and GANs, are both directed models, and both have been successfully adapted to semi-supervised learning \cite{catgan,improved-gan,bad-gan,S2016Ladder}. 
In contrast, undirected generative models (also known as random fields \cite{koller2009probabilistic}, energy-based models \cite{energy-based}), e.g. Deep Boltzmann Machines (DBMs) \cite{dbm}, received less attentions with slow progress.
This is presumably because fitting undirected models is more challenging than fitting directed models \cite{koller2009probabilistic}.
In general, calculating the log-likelihood and its gradient is analytically intractable, because this involves evaluating the normalizing constant (also called the partition function in physics) and, respectively, the expectation with respect to the model distribution.

In this paper, we are interested in building deep random field models for semi-supervised tasks. Note that current SSL methods with VAEs and GANs have some undesirable problems.
First, VAEs are characterized by the mode covering behavior, where they tend to generate blurry samples when applied to natural images.
GANs suffer from the mode missing problem, where the resulting model is unable to capture all the variability in the true data distribution.
There are efforts to combine VAEs and GANs to balance mode covering and mode missing \cite{aae,avb}.
Second, in the contex of SSL, VAEs performs worse than with GANs.
For adaptation of GANs to SSL using the $\left(K+1\right)$-class discriminative objective, it is observed that good semi-supervised classification performance and good generation cannot be obtained at the same time \cite{improved-gan}; and it is further analyzed that good semi-supervised learning indeed requires a bad generator \cite{bad-gan}. 
This conflict of good classification and good generation is embarrassing and in fact obviates the original idea of generative model based SSL - successful generative training, which indicates good generation, provides regularization for finding good classifiers.
This problem is addressed in \cite{triple-gan} by using the $K$-class classifier and utilizing three networks with additional self-ensembling, and further in \cite{trianglegan} by utilizing four networks.

Note that for SSL to make up for the lack of labeled training data, good matching of model assumption with the structure of data is critical.
Random field modeling provide an interesting, alternative family of model spaces for exploring new generative models for SSL, outside of deep directed generative models, to address the above two problems - balancing mode covering and mode missing, and alleviating the conflict of generation and classification.

To build deep random field models for SSL, we note that recently, a new learning algorithm called Joint-stochastic-approximation (JSA) is developed initially for a broad class of directed generative models \cite{jsa} and further applied successfully to learning with random field language models \cite{asru}.
The key idea is to pair the target random field with an auxiliary directed generative model, which approximates the target random field but is easy to do sampling and point-wise likelihood evaluation.
Then we can maximize the target data log-likelihood and simultaneously minimizes the inclusive KL divergence between the target random field and the directed generator.
As shown in \cite{jsa}, this optimization problem can be solved based on stochastic approximation (SA) theory \cite{SA51}, by stochastically solving a system of simultaneous equations which have the form of expectations being equal to zeros.
We refer to this type of random field modeling and learning by JSA random fields, or JRFs for short. 

Remarkably, at a high-level, an independent work in \cite{coopnets} uses the same type of model optimization method as in \cite{asru}, but applies for unsupervised learning of images with some low-level differences, which we will detail in Related work.
Inspired by these successes of unsupervised learning with JRFs in \cite{coopnets,asru}, we examine its adaptation to SSL in this paper, which promisingly addresses the above two problems with existing directed generative modeling techniques for SSL.
Particularly, in this paper, as we mainly demonstrate the JRF methodology in semi-supervised image classification, we utilize Langevin sampling as used in \cite{coopnets} for image models, instead of Metropolis Independence Sampling (MIS) as used in \cite{asru} for language modeling.

The contributions of this work can be summarized as:

(1) We formally introduce Joint-stochastic-approximation random fields (JRFs), - a new family of algorithms which enables building deep undirected generative models for semi-supervised tasks, and show its theoretical consistency in the nonparametric limit.

(2) Synthetic experiments are given to help us analyze JRFs in-depth and understand their capabilities in balancing mode covering and mode missing, and alleviating the conflict of generation and classification.

(3) We empirically show that JRFs achieve state-of-the-art classification performances in semi-supervised tasks over the widely adopted image datasets - MNIST \cite{mnist}, SVHN \cite{svhn} and CIFAR-10 \cite{cifar}, and at the same time perform excellent image generation. 
To the best of our knowledge, this is the first demonstration that deep random field models are successfully applied in the challenging semi-supervised tasks.

\section{Related work}

\subsection{Semi-supervised learning}
In this work, we are mainly concerned with semi-supervised learning with deep generative models (DGMs).
The main idea is that generative training over unlabeled data provides regularization for finding good classifiers \cite{zhu05}.
From the perspective of regularization, virtual adversarial training (VAT) seeks virtually adversarial samples to smooth the output distribution of the classifier \cite{vat}, temporal ensembling \cite{temperal} and mean teacher \cite{meanteacher} maintain running averages of label predictions and model weights respectively for regularization.
These SSL methods also achieve good results. It can be seen that these SSL methods utilize different regularization from SSL with generative models. Their combination could yield further performance improvement in practice.

Early forms of semi-supervised learning date back to the practice of unsupervised pre-training or feature learning followed by supervised learning.
Restricted Boltzmann Machines (RBMs), a particular type of un-directed models, have once been popular for this purpose.
Classification RBMs are among one of the early efforts, which study SSL under the recent setting of performing unsupervised learning and supervised learning simultaneously, but only with a single generative model of observations and labels \cite{class-rbm}.
Currently most state-of-the-art SSL methods are based on deep directed generative models.
The most prominent DGM techniques that have been successfully adapted to semi-supervised learning, VAEs and GANs, are both directed models.
There are some studies that connects GANs with energy-based models \cite{ebgan,coopnets,kim16}, which will be introduced in the next Sub-section. SSL with deep random fields remains under-explored.

\subsection{Learning with random fields}

There is an extensive literature devoted to maximum likelihood (ML) learning of random fields (RFs), which are sometimes called energy-based models.
An important class of RF learning methods is stochastic approximation methods \cite{SA51}, which approximates the model expectations by Monte Carlo sampling for calculating the gradient. 
Basically, SA training iterates Monte Carlo sampling and SA update of parameters.
The classic algorithm, initially proposed in \cite{younes1989parametric}, is often called stochastic maximum likelihood (SML).
In the literature on training restricted Boltzmann machines (RBMs), SML is also known as persistent contrastive divergence (PCD) \cite{tieleman2008training} to emphasize that the Markov chain is not reset between parameter updates.

Roughly speaking, there are two types of deep random fields in the literature. Those with multiple stochastic hidden layers such as deep belief networks (DBNs) \cite{dbn} and deep Boltzmann machines (DBMs) \cite{dbm} involve very difficult inference and learning, which limits their use in SSL beyond of the form of pre-training.
Another type is to directly define an energy function through a deep neural network. In this case, the layers of the network do not represent latent variables but rather are deterministic transformations of input observations. 

The second type of deep random fields, which is also employed in this work, has been proposed several times in different contexts. They are once called deep energy models (DEMs) in \cite{ng11}, descriptive models in \cite{descriptor,coopnets}, generative ConvNet in \cite{wyn15}, and neural random fields in \cite{asru}.
The classic learning algorithm is still SML.
In \cite{ng11}, contrastive divergence \cite{cd} is used, which, in contrast to PCD, is biased in theory since it performs Monte Carlo sampling from the training observations.
In \cite{wyn15} Hamiltonian Monte Carlo (HMC) is used for Monte Carlo sampling.
A recent progress as studied in \cite{kim16,jsa,coopnets} is to pair the target random field with an auxiliary directed generative model, which approximates the target random field but is easy to do sampling and point-wise likelihood evaluation.
In addition to maximizing the target data log-likelihood, \cite{kim16} simultaneously minimizes the \emph{exclusive} KL divergence between the target random field and the directed generator, while \cite{jsa,coopnets} minimizes the \emph{inclusive} KL divergence.
Remarkably, the exclusive KL divergence includes a reconstruction term and an entropy term, which resembles the expression of the exclusive KL divergence in the variational learning. The entropy term is analytically intractable, so an ad hoc approximation is used in \cite{kim16} without strict justification.
In contrast, the inclusive KL divergence as employed by JSA does not have such annoying entropy term.

In JSA learning, the auxiliary directed generator serves as the proposal for constructing the MCMC operator for the target random field, for which there are a number of choices. 
Metropolis Independence Sampling (MIS) is used in \cite{jsa,asru} which is appropriate for handling discrete observations, while Langevin sampling is used in \cite{coopnets} which is more efficient than MIS when handling continuous observations.

There are some studies that connects GANs with energy-based models. The energy-based GAN (EBGAN) model \cite{ebgan}, which proposes to view the discriminator as an energy function, is shown to stabilize the training and generate high-resolution images.
The further work in \cite{cali} aims to address the inability of GANs to provide sensible energy estimates for samples. It connects \cite{ebgan} and \cite{kim16}, and show another two approximations for the entropy term. None of the studies examine their methods or models for SSL, except in EBGAN which performs moderately.

\section{Method}
\subsection{Background}
Our method is an application of the stochastic approximation (SA) framework \cite{SA51}, which basically provides a mathematical framework for stochastically solving a root finding problem, which has the form of expectations being equal to zeros.
Suppose that the objective is to find the solution $\lambda^*$ of $f(\lambda) = 0$ with
\begin{equation}
\label{eq:SA}
f(\lambda) = E_{z \sim p(\cdot; \lambda) } [ F(z;\lambda) ],
\end{equation}
where $\lambda \in R^d$ is a parameter vector of dimension $d$, and $z$ is an observation from a probability distribution $p(\cdot; \lambda)$ depending on $\lambda$, and $F(z;\lambda) \in R^d $ is a function of $z$.
Given some initialization $\lambda^{(0)}$ and $z^{(0)}$, a general SA algorithm iterates as follows.
\begin{enumerate}
	\item Generate $z^{(t)} \sim K_{\lambda^{(t-1)}}(z^{(t-1)},\cdot)$, a Markov transition kernel that admits $p(\cdot; \lambda^{(t-1)})$ as the invariant distribution.
	\item Set $\lambda^{(t)} = \lambda^{(t-1)} + \gamma_t F(z^{(t)};\lambda^{(t-1)}) $, where $\gamma_t$ is the learning rate.
\end{enumerate}

During each SA iteration, it is possible to generate a set of multiple observations $z$ by performing the Markov transition repeatedly 
and then use the average of the corresponding values of $F(z;\lambda)$ for updating $\lambda$, which is know as SA with multiple moves \cite{pami}.
This technique can help reduce the fluctuation due to slow-mixing of Markov transitions. 
The convergence of SA has been studied under various regularity conditions, e.g. satisfying that $\sum_{t=0}^\infty \gamma_t = \infty$ and $\sum_{t=0}^\infty \gamma_t^2 < \infty$. In practice, we can set a large learning rate at the early stage of learning and decrease to $1/t$ for convergence.

\subsection{Joint-stochastic-approximation Learning for Random Fields}

In the following we present the JSA algorithm for learning random fields. Consider a random field for modeling observation $x$ with parameter $\theta$:
\begin{displaymath}
p_{\theta}(x)=\frac{1}{Z(\theta)} \exp\left[  u_{\theta}(x) \right] 
\end{displaymath}
where $Z(\theta)=\int\exp(f(x;\theta))dx$ is the normalizing constant, $u_{\theta}(x)$ is the potential function.
It is usually intractable to maximize the data log-likelihood $log p_\theta(\tilde{x})$ for observed $\tilde{x}$, since the gradient involves expectation w.r.t. the model distribution, as shown below:
\begin{displaymath}
\frac{\partial }{\partial \theta}\log{p}_{\theta}(\tilde{x})=\frac{\partial}{\partial \theta}u_{\theta}(\tilde{x})-E_{p_\theta(x)}\left[\frac{\partial}{\partial \theta}u_{\theta}(x)\right]
\end{displaymath}

We can pair the target random field $p_{\theta}(x)$ with an auxiliary directed generative model $q_\phi(x,h) \triangleq q(h)q_\phi(x|h)$, so that by jointly training the two model, the marginal $q_\phi(x)$ approximates the target random field.
This idea has been studied in \cite{coopnets,kim16} by minimizing the \emph{exclusive} KL divergence $KL(q_\phi(x)||p_\theta(x))$, which includes an annoying entropy term.
The distinctive key idea of JSA learning is that 
in addition to maximizing w.r.t. $\theta$ the data log-likelihood, it simultaneously minimizes w.r.t. $\phi$ the \emph{inclusive} KL divergence $KL(p_\theta(x)||q_\phi(x))$ between the target random field and the directed generator. And fortunately, we can use the SA framework to solve the optimization problem.

Suppose that data $\mathcal{D} = \left\lbrace \tilde{x}_1, \cdots, \tilde{x}_n \right\rbrace $, which consists of $n$ observations drawn from the true but unknown data distribution $p_0(\cdot)$ with support $\mathcal{X} $.
$\tilde{p}(\tilde{x}) \triangleq \frac{1}{n} \sum_{k=1}^{n} \delta(\tilde{x} - \tilde{x}_n)$ denotes the empirical distribution.
Then we can formulate the maximum likelihood learning as optimizing
\begin{equation}
\label{eq:jrf_unsup_obj}
\left\{
\begin{split}
& \min_{\theta} KL\left[  \tilde{p}(\tilde{x}) || p_\theta(\tilde{x}) \right] \\
& \min_{\phi} KL\left[  p_\theta(x) || q_\phi(x) \right] \\
\end{split}
\right.
\end{equation}
By setting the gradients to zeros, the above optimization problem can be solved by finding the root for the following system of simultaneous equations:
\begin{equation}
\label{eq:jrf_unsup_gradient}
\left\{
\begin{split}
& E_{\tilde{p}(\tilde{x})}\left[\frac{\partial}{\partial\theta} logp_\theta(\tilde{x})\right]\\
=&E_{\tilde{p}(\tilde{x})}\left[\frac{\partial}{\partial\theta}u_\theta(\tilde{x})\right]-E_{p_\theta(x)}\left[\frac{\partial}{\partial\theta}u_\theta(x)\right]=0\\
&E_{p_\theta(x)}\left[ \frac{\partial}{\partial\phi} logq_\phi(x)\right]\\
=&E_{p_\theta(x) q_\phi(h|x)}\left[ \frac{\partial}{\partial\phi} logq_\phi(x,h)\right]=0\\
\end{split}
\right.
\end{equation}
It can be shown that Eq.(\ref{eq:jrf_unsup_gradient}) exactly follows the form of Eq.(\ref{eq:SA}), so that we can apply the SA algorithm to find its root and thus solve the optimization problem Eq.(\ref{eq:jrf_unsup_obj}). 

\begin{prop}
	If Eq.(\ref{eq:jrf_unsup_gradient}) is solvable, then we can apply the SA algorithm to find its root.
\end{prop}

\begin{proof}
	This can be readily shown by recasting  Eq.(\ref{eq:jrf_unsup_gradient}) in the form of $f(\lambda) = 0$, with $\lambda \triangleq (\theta, \phi)^T$, $z \triangleq (\tilde{x}, x, h)^T$, $p(z; \lambda) \triangleq \tilde{p}(\tilde{x}) p_\theta(x) q_\phi(h|x)$, and

$
	F(z; \lambda) \triangleq  \left( \begin{array}{c}
	\frac{\partial}{\partial\theta}u_\theta(\tilde{x}) - \frac{\partial}{\partial\theta}u_\theta(x)\\
	\frac{\partial}{\partial\phi} logq_\phi(x,h)
	\end{array} \right).
$
\end{proof}

\begin{algorithm}[tb]
	\caption{Sample revision $x'\to x$ }\label{alg:revision}
	\begin{algorithmic}	
		\STATE SGLD:
		\STATE \textbf{Input:} $M, \gamma, x', \delta$
		\STATE \textbf{Output:} $x=x^{(M)}$
		\STATE set $x^{(0)}=x'$
		\FOR {$t=1,2,...,M$}
		\STATE $x^{(t)}=x^{(t-1)}+\gamma\frac{\partial }{\partial x}\log{p_\theta(x^{(t-1)})}+\delta U$, $U\sim N(0,I)$
		\ENDFOR
		\STATE
		\STATE SGHMC:
		\STATE \textbf{Input:} $M, \beta, \eta, x', \delta$
		\STATE \textbf{Output:} $x=x^{(M)}$
		\STATE set $x^{(0)}=x', v^{(0)}=0$
		\FOR {$t=1,2,...,M$}
		\STATE $v^{(t)}=\beta v^{(t-1)}+\eta\frac{\partial }{\partial x}\log{p_\theta(x^{(t-1)})}+\delta U$, $U\sim N(0,I)$
		\STATE $x^{(t)}=x^{(t-1)}+v^{(t)}$
		\ENDFOR
		
	\end{algorithmic}
\end{algorithm}

To apply the SA algorithm, we need to contruct a Markov transition kernel $K_{\lambda}(z^{(t-1)},\cdot)$ that admits $p(z; \lambda)$ as the invariant distribution. 
It is straightforward to draw from $\tilde{p}(\tilde{x})$, and the problem is to construct the Markov kernel $K_{\theta,\phi}((x^{(t-1)},h^{(t-1)}),\cdot)$ that admits $p_\theta(x) q_\phi(h|x)$ as the invariant distribution.
There are many options.
Inspired by using Langevin sampling with finite steps starting from the proposal defined by the generator in \cite{coopnets},
we employ SGLD (stochastic gradient Langevin dynamics) \cite{sgld} and SGHMC (Stochastic Gradient Hamiltonian Monte Carlo) \cite{sghmc} to perform the Markov transition.
Particularly, we use the Metropolis independence sampler (MIS), with $p_\theta(x) q_\phi(h|x)$ as the target distribution. The proposal is defined as follows:
\begin{enumerate}
	\item Do ancestral sampling by the generator, namely first drawing $h \sim p(h)$, and then drawing $x' \sim q_\phi(x'|h)$,
	\item Starting from $x'$, execute a sample revision process to obtain $x$, as shown in Alg.\ref{alg:revision}.
\end{enumerate}

Intuitively, the generator gives a proposal $(h,x')$, and then the system follows the gradient of $p_\theta(x)$ of the random field to revise $x'$ to $x$. The gradient term pulls samples moving to low energy region of the random field, while the noise term brings randomness.
According to \cite{sgld}, the discretization error of Langevin dynamics will be negligible so that the rejection probability in MIS will approach zero. So we may simply ignore this rejection step and always accept $(h,x)$ as samples from $p_\theta(x) q_\phi(h|x)$.

Since the parameters of the target and auxiliary models are jointly optimized based on the SA framework, the above method is referred to as JSA learning.
The resulting model is called JSA random field, or JRF for short.
Next, we examine SSL with JRFs.

\subsection{Semi-supervised learning with JRFs}
In semi-supervised tasks, we consider a random field for joint modeling of observation $x$ and class label $y$:
\begin{displaymath}
p_{\theta}(x,y)=\frac{1}{Z(\theta)} \exp\left[  u_{\theta}(x,y) \right] 
\end{displaymath}
Suppose that among the data $\mathcal{D} = \left\lbrace \tilde{x}_1, \cdots, \tilde{x}_n \right\rbrace $, only a small subset of the observations, for example the first $m$ observations, have class labels, $m \ll n$.
Denote these labeled data as $\mathcal{L} = \left\lbrace(\tilde{x}_1,\tilde{y}_1), \cdots, (\tilde{x}_m,\tilde{y}_m) \right\rbrace $, with $\tilde{p}(\tilde{x},\tilde{y})$ representing the  empirical distribution.
Then we can formulate the semi-supervised learning as jointly optimizing
\begin{equation}
\label{eq:jrf_semi_obj}
\left\{
\begin{split}
& \min_{\theta} KL\left[  \tilde{p}(\tilde{x}) || p_\theta(\tilde{x}) \right]
- \alpha \sum_{(\tilde{x},\tilde{y}) \sim \mathcal{L}} log p_\theta(\tilde{y}|\tilde{x}) \\
& \min_{\phi} KL\left[  p_\theta(x) || q_\phi(x) \right] \\
\end{split}
\right.
\end{equation}
By setting the gradients to zeros, the above optimization problem can be solved by finding the root for the following system of simultaneous equations:
\begin{equation}
\label{eq:jrf_semi_gradient}
\left\{
\begin{split}
& E_{\tilde{p}(\tilde{x})}\left[\frac{\partial}{\partial\theta} logp_\theta(\tilde{x})\right]
+ \alpha \sum_{(\tilde{x},\tilde{y}) \sim \mathcal{L}} \frac{\partial}{\partial\theta} log p_\theta(\tilde{y}|\tilde{x})\\
=&E_{\tilde{p}(\tilde{x})}\left[\frac{\partial}{\partial\theta}u_\theta(\tilde{x})\right]-E_{p_\theta(x)}\left[\frac{\partial}{\partial\theta}u_\theta(x)\right]\\
+& \alpha \sum_{(\tilde{x},\tilde{y}) \sim \mathcal{L}} \frac{\partial}{\partial\theta} log p_\theta(\tilde{y}|\tilde{x})=0\\
&E_{p_\theta(x)}\left[ \frac{\partial}{\partial\phi} logq_\phi(x)\right]\\
=&E_{p_\theta(x) q_\phi(h|x)}\left[ \frac{\partial}{\partial\phi} logq_\phi(x,h)\right]=0\\
\end{split}
\right.
\end{equation}
where $u_\theta(x) \triangleq log \sum_y \exp\left[  u_{\theta}(x,y) \right]$, with abuse of notation.
Similarly, it can be shown that Eq.(\ref{eq:jrf_semi_gradient}) exactly follows the form of Eq.(\ref{eq:SA}), so that we can apply the SA algorithm to find its root and thus solve the optimization problem Eq.(\ref{eq:jrf_semi_obj}).
The SA algorithm with multiple moves for SSL is shown in Algorithm 2.

The random field potential function $u_\theta(x,y)$ is implemented by a neural network with $x$ as the input and the output size is equal to the number of class labels.
The directed generator is implemented as:
$q_{\phi}(x|h)=G(h)+\sigma N(0,I)= N(G(h),\sigma^2I)$. 
$\sigma$ is the standard deviation (std) of the guassian noise. $G(h)$ is implemented as a neural network.

\begin{algorithm}[tb]
	\caption{Semi-supervised learning with JRFs}
	\label{alg:JRF}
	\begin{algorithmic}
		\REPEAT
		\STATE \underline{Monte Carlo sampling:}
		
		Draw a unsupervised minibatch $\mathcal{U} \sim \tilde{p}(\tilde{x}) p_\theta(x) q_\phi(h|x))$ and a supervised minibatch $\mathcal{S} \sim \tilde{p}(\tilde{x},\tilde{y}) p_\theta(x) q_\phi(h|x))$;
		
		\STATE \underline{SA updating:}
		
		Update $\theta$ by ascending:

		$\frac{1}{|\mathcal{U}|} \sum_{(\tilde{x},x,h) \sim \mathcal{U}}
		\left[\frac{\partial}{\partial\theta}u_\theta(\tilde{x}) - \frac{\partial}{\partial\theta}u_\theta(x) \right] $
		$
		+\frac{1}{|\mathcal{S}|} \sum_{(\tilde{x},x,h) \sim \mathcal{S}}
		\left[\frac{\partial}{\partial\theta}u_\theta(\tilde{x}) - \frac{\partial}{\partial\theta}u_\theta(x) + \alpha \frac{\partial}{\partial\theta} log p_\theta(\tilde{y}|\tilde{x})\right]
		$

		Update $\phi$ by ascending:
		
		$
		\frac{1}{|\mathcal{U}|} \sum_{(\tilde{x},x,h) \sim \mathcal{U}}
		\frac{\partial}{\partial\phi} logq_\phi(x,h) $
		$+ \frac{1}{|\mathcal{S}|} \sum_{(\tilde{x},x,h) \sim \mathcal{S}}
		\frac{\partial}{\partial\phi} logq_\phi(x,h)  
		$
		
		\UNTIL{convergence}
	\end{algorithmic}
\end{algorithm}

\subsection{Regularization}
Apart from the basic loss of JRFs, there are some regularization loss that is helpful to guide the direction of training.

\textbf{Confident loss} By maximizing $p_\theta(x_u)$ for unlabed $x_u$, the random field is not forced to predict a dominant class.
Considering $p_\theta(x,y)$ fit $p(x,y)$ well should give distinct prediction for $x_u$,
we incorporate an unbiased loss to minimize entropy of $p_\theta(y|x_u)$ from \cite{catgan}:
\begin{equation}
\label{eq:rc}
\begin{aligned}
R_c&=H(p_\theta(y|x_u))\\
&=\frac{1}{n_u}\sum_{i=1}^{n_u}[\sum_yp_\theta(y|x_u^{(i)})\log{p_\theta(y|x_u^{(i)})}]
\end{aligned}
\end{equation}
Note that adding $R_c$ to $L_D$ doesn't influence the optimal of $L_{Du}$ and $L_{Dl}$.
It performs ranking the optimal models trained with basic loss.
By leveraging $R_c$, the descriptor will predict classes of $x_u$ more confidently, thus called confident loss \cite{triple-gan}.

\textbf{Self-normalization loss} For random fields, the probability is only decided by relative value of energy, but not the absolute value.
We thus could control the absolute energy of training data $Z(x;\theta)=\sum_y\exp(u_\theta(x,y))$, 
where $Z(x;\theta)$ is the normalization constant for $p_\theta(y|x)$. 
\begin{equation}
\begin{cases}
Z(x;\theta)=\sum_y\exp(u_\theta(x,y))\\
R_s=(\frac{1}{n_u}\sum_{i=1}^{n_u}\log{Z(x;\theta)})^2
\end{cases}
\end{equation}
By leveraging $R_s$, the average energy of a batch unlabeled data will be attracted to $0$.
We use the mean statistic of a batch to perform soft constraint, while there still exists a bit bias from limited batch size.
\begin{figure*}[htb]
	\subfigure[training set]{
	\begin{minipage}{0.23\textwidth}
		\flushleft  
		\includegraphics[width=\textwidth]{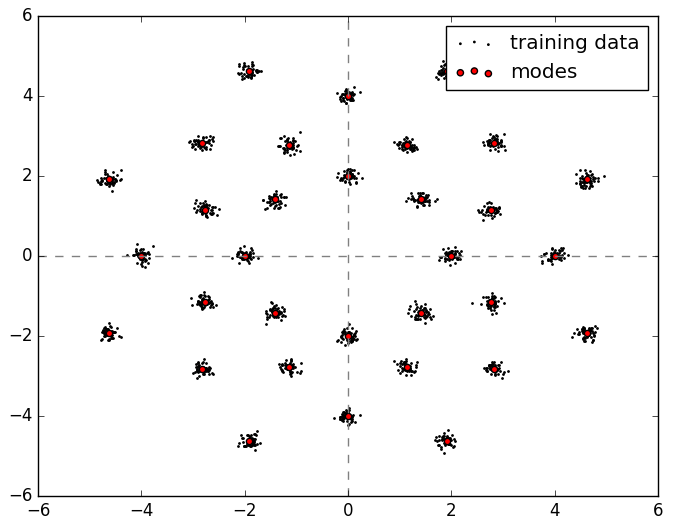}  
	\end{minipage}}
	\subfigure[GAN w/o FM generation]{
	\begin{minipage}{0.23\textwidth}  
\flushleft
		\includegraphics[width=\textwidth]{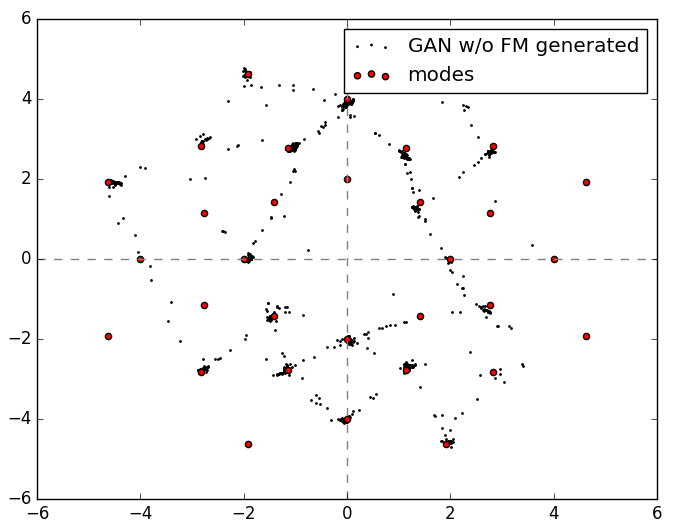}  
	\end{minipage}}
	\subfigure[GAN with FM generation]{  
	\begin{minipage}{0.23\textwidth}  
		\flushleft  
		\includegraphics[width=\textwidth]{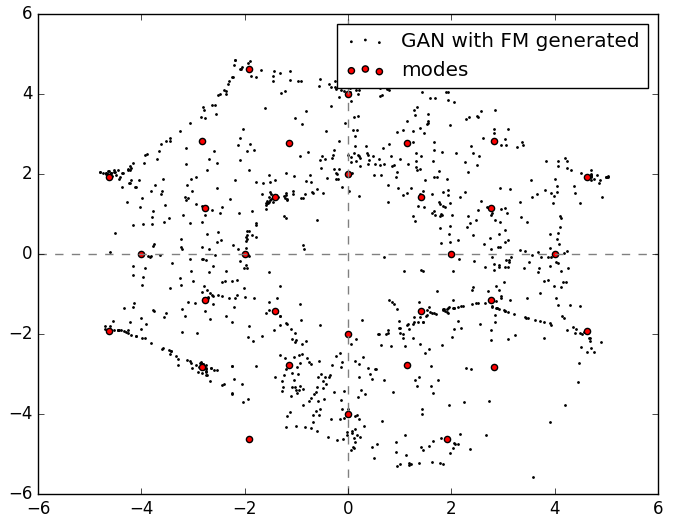}    
	\end{minipage}}
	\subfigure[EBGMs generation]{  
	\begin{minipage}{0.23\textwidth}  
		\flushleft  
		\includegraphics[width=\textwidth]{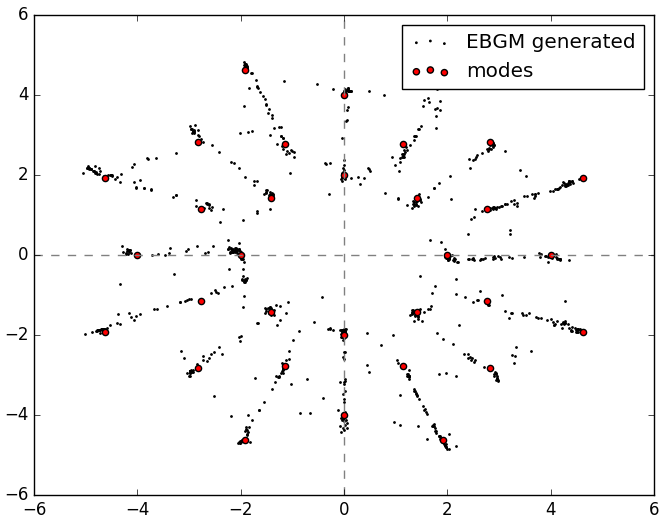}    
\end{minipage}}
	\subfigure[JRFs generation]{  
	\begin{minipage}{0.23\textwidth}  
		\flushleft  
		\includegraphics[width=\textwidth]{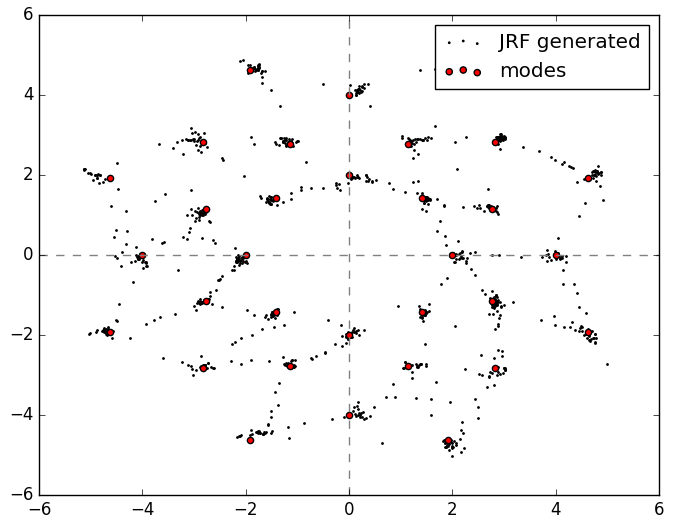}    
	\end{minipage}}
\hspace{6pt}
	\subfigure[JRFs revision]{  
	\begin{minipage}{0.23\textwidth}  
		\centering  
		\includegraphics[width=\textwidth]{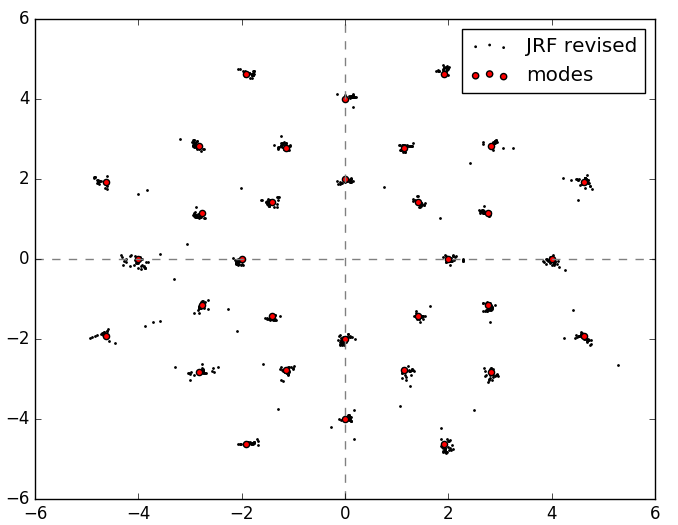}    
	\end{minipage}}
\hspace{10pt}
\subfigure[EBGMs energy density]{  
	\begin{minipage}{0.22\textwidth}  
		\centering  
		\includegraphics[width=1\textwidth]{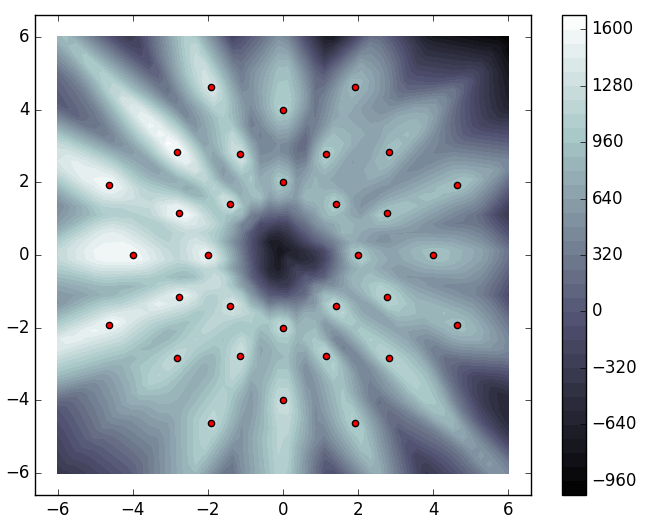}    
\end{minipage}}
\hspace{10pt}
	\subfigure[JRFs energy density]{  
	\begin{minipage}{0.22\textwidth}  
		\centering  
		\includegraphics[width=1\textwidth]{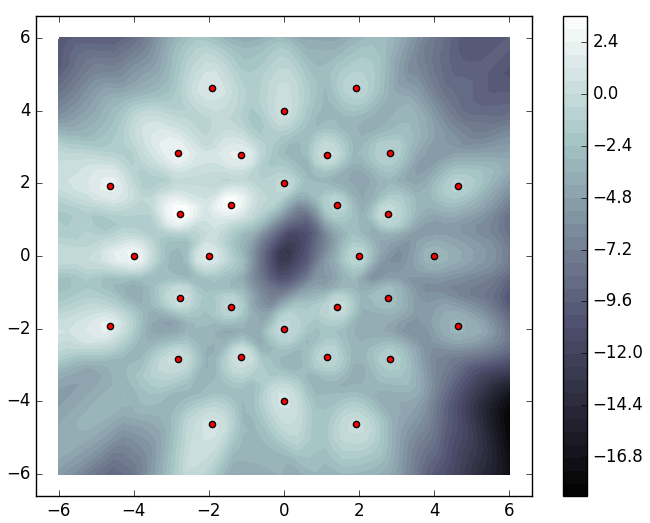}    
	\end{minipage}}
	\caption{Figures of the first toy experiment. (a) data distribution of training set. 
	(b)stochastic generation of GAN without feature matching.
	(c)stochastic generation of GAN with feature matching.
	(d)stochastic generation of EBGMs.
	(e)stochastic generation of JRFs.
	(f)samples of JRFs from a revision process on stochastic generation.
	(g)the learned energy density of the random field of EBGMs.
	(h)the learned energy density of the random field of JRFs.
	The red dots represent the modes of training distribution.
	Each generation contains 1,000 samples.
	For (g)(h), white represents low energy and black for high energy.
	The energy density of JRFs matches $p(x)$ better, while of EBGMs shows trailers around modes.
}
\label{fig:toy}
\end{figure*}  

\begin{figure*}[htb]
	\subfigure[Improved-GAN on SVHN]{
		\begin{minipage}{0.24\textwidth}
			\centering  
			\includegraphics[width=\textwidth]{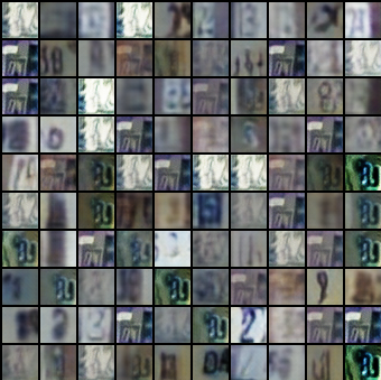}  
	\end{minipage}}
	\subfigure[semi-JRFs on SVHN]{
		\begin{minipage}{0.24\textwidth}  
			\centering  
			\includegraphics[width=\textwidth]{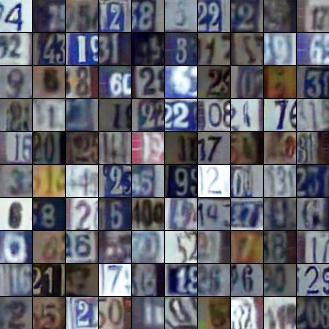}  
	\end{minipage}}
	\subfigure[Improved-GAN on CIFAR]{  
		\begin{minipage}{0.24\textwidth}  
			\centering  
			\includegraphics[width=\textwidth]{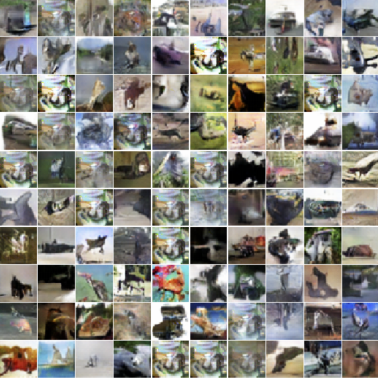}    
	\end{minipage}}
	\subfigure[semi-JRFs on CIFAR]{  
		\begin{minipage}{0.24\textwidth}  
			\centering  
			\includegraphics[width=\textwidth]{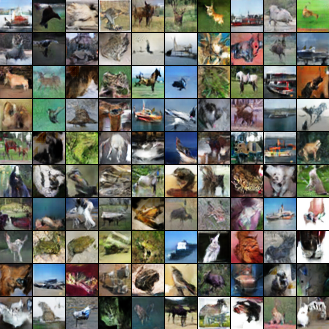}    
	\end{minipage}}
	\caption{Comparing images generated by Improved-GAN and semi-JRFs on SVHN and CIFAR-10. Improved-GAN generates collapsed and strange samples, while semi-JRFs preform diverse and realistic.
}
\label{fig:generate}
	\vskip -0.1in
\end{figure*}  
\section{Experiments}
\begin{table}[tbp]
		\vskip -0.2in
	\centering  
		\caption{Statistic of the first toy experiment.
		"Covered modes" means how many modes are covered by generated samples.
		"realistic ratio" means the proportion of generated samples covering a mode.
		JRFs revised samples are from a revision process on JRFs generated samples.}
	\label{tab:toy}
	\begin{tabular}{lcc}  
		\hline
		Methods &covered modes  &realistic ratio \\  
		\hline
		GANs w/o FM&$22.25\pm1.54$ &$0.90\pm0.01$\\         
		GANs with FM &$20.43\pm1.40$ &$0.42\pm0.04$\\        
		EBGMs &$28.14\pm0.68$ &$0.73\pm0.03$\\  
		JRFs generated    &$29.52\pm0.54$ &$0.84\pm0.01$\\
		JRFs revised &$30.75\pm0.43$ &$0.97\pm0.01$\\
		\hline
	\end{tabular}
	\vskip -0.2in
\end{table}
\begin{table*}[tbp]
	\centering  
	\caption{Comparison with state-of-the-art methods on three benchmark datasets,
		only methods without data augmentation contained.}
	\label{main result}
	\begin{tabular}{lccc}  
		\hline
		Methods &MNIST(\# error)  &SVHN(\% error) &CIFAR-10(\% error)\\  
		\hline
		CatGAN\cite{catgan} &$191\pm10$ &- &$19.58\pm0.46$\\         
		SDGM\cite{adgm} &$132\pm7$ &$16.61\pm0.24$ &-\\        
		Ladder network\cite{ladder} &$106\pm37$ &- &$20.40\pm0.47$\\
		ADGM\cite{adgm} &$96\pm2$ &22.86 &-\\
		FM\cite{improved-gan} &$93\pm6.5$ &$8.11\pm1.3$ &$18.63\pm2.32$\\
		ALI\cite{ali} &- &$7.42\pm0.65$ &$17.99\pm1.62$\\
		Triple-GAN\cite{triple-gan} &$91\pm58$ &$5.77\pm0.17$ &$16.99\pm0.36$\\
		VAT small\cite{vat} &$136$ &6.83 &14.87\\
		BadGAN\cite{bad-gan} &$79.5\pm9.8$ &$4.25\pm0.03$ &$14.41\pm0.30$\\
		\hline
		semi-JRFs &$97\pm10$ &$5.84\pm0.15$ &$15.51\pm0.36$
	\end{tabular}
\end{table*}
\subsection{Case study}
Considering that generative SSL methods often develop from unsupervised learning methods of DGMs, we first evaluate the unsupervised learning of JRFs, and some other generative models for comparison.
For directed generative models,
as in the spotlight, GANs are good at generating realistic data while observed being troubled with mode missing problem. 
Recently GAN-based methods \cite{improved-gan,bad-gan} reach a remarkable performance in SSL, while feature matching (FM for short) loss is employed to train the generator.
Considering undirected generative models,
\cite{kim16} propose a method similar to JRFs combining energy-based models and generative models, and we call this method EBGMs for referring simplicity. 
The main difference between EBGMs and JRFs is that EBGMs directly use samples of the generative model as approximate samples of random field (or energy-based model), and minimize a reversed KL divergence between the two models compared to JRFs, thus need to deal with an annoying entropy term.

We design a toy experiment to compare the representative methods mentioned above in unsupervised learning setting, by showing how well the generative model learns for GANs with or without FM, and how well the both models learn for EBGMs and JRFs.
The dataset is generated from a 2D gaussian mixture with 32 mixture components uniformly distributed on four concentric circles, containing 1,600 training examples (see Fig.\ref{fig:toy}(a)).
We use the Adam \cite{adam} optimizer to train each model with mini-batch stochastic gradient decent.
The same structure is used for each model:
2-50-50-2 with batch normalization\cite{batchnorm} for generative model, 
and 2-100-100-1 with wight normalization\cite{weightnorm} for random field and discriminator of GANs.

Fig.\ref{fig:toy} visually illustrates the performance of each method.
Considering that JRFs in fact use revised samples (first stochastically generated, then passed through a revision process) for training, we illustrate both generated and revised samples.
Notice that it is hard to perform evaluation on learned distribution for discriminator of GANs, while energy density representing unnormalized likelihood for random field can be evaluated, which is a superiority of undirected generative models.
Besides vision, we further rely on statistic for exact comparison. 
Aiming to quantify the mode missing phenomenon and generation quality, we execute the following estimation:

1) stochastically generate 100 samples

2) a mode is covered (or not missing) by defining as there exists generated sample locating closely to the mode less than a threshold, and a sample is realistic by defining as it covers a mode

3) count how many modes are covered and the proportion of realistic samples by averaging over 100 times generation

4) count the mean and standard deviation (std) across 10 respectively trained models

Tab.\ref{tab:toy} summarizes the statistic of samples from different methods.
From Fig.\ref{fig:toy} and Tab.\ref{tab:toy}, we observer as follows:
\begin{itemize}
\item GANs with or without FM are troubled in mode missing problem while undirected generative methods JRFs and EBGMs perform superior
\item GANs with FM have worse generation quality for generating many unrealistic samples out of modes, as a result of weak distribution matching\cite{bad-gan}
\item JRFs outperform EBGMs among generation quality and energy density learning, thus show better fitting across distribution of training dataset, the random field and the generative model
\item For JRFs, after a revision process, the generated data becomes more like
samples from training dataset (the random field fits $p(x)$ well), which shows the working mechanism and effect of revision process
\end{itemize}

In summary, this experiment illustrates that in JRFs the generative model properly draws samples as from the random field distribution (the generative model fits the random field well) and the random field fits the empirical distribution well.
Compared to GANs, JRFs have better generation capacity and avoid the mode missing problem, which suggests the generative model fits the empirical distribution better through a well-learned random field. 
Compared to the similar method EBGMs, JRFs exhibit more sophisticated matching across the empirical distribution, the random field and the generative model,
which probably shows the advantage of employing MCMC to obtain samples of the random field.
Compared to the worse generation by employing FM, JRFs would potentially perform successful generative training in SSL with theoretical consistency.

\subsection{Main results}

We evaluate the performance of the proposed method semi-JRFs for classification tasks on three widely adopted datasets: MNIST, SVHN and CIFAR-10.
MNIST consists of 60,000 training examples and 10,000 for testing, and each example is a handwriting digit image of size $28\times28$.
SVHN consist of 73,257 training examples and 26,032 testing examples, and each is a digit image of size $3\times32\times32$ with diverse background.
CIFAR-10 consists 50,000 training examples and 10,000 for testing, each is a natural image of size $3\times32\times32$.
There are 10 general object classes in CIFAR-10, like cat and ship.

In setting of SSL, there are limited labeled training examples available.
For general setting, we randomly choose 100 class-balanced labeled examples for MNIST, 1000 for SVHN, 4000 for CIFAR-10, and the rest training examples are unlabeled.
The results are averaged over 10 times training with different randomly chosen labeled data, and are estimated on the standard testing set.

These are some details in training on the three datasets. 
We use the same structure of Improved-GAN \cite{improved-gan}, employing weight normalization for random field and batch normalization for generator.
We optimize all networks with RMSProp where the decay rate is $0.9$.
The revision process contains 20 transition steps for MNIST, and 10 for SVHN and CIFAR-10.
During training, we in fact ignore the white noise of generated samples and revision process.
Without the noise, the convergence of JRFs is not essentially influenced but reached faster, and due to the plenty of modes of the three datasets, the generation still performs highly diverse. 
Confident loss $R_c$ is employed for training on MNIST and SVHN, and Self-normalization loss $R_s$ is used for CIFAR-10.
Combining with these regularization loss, semi-JRFs perform better in SSL empirically.

We compare the results on MNIST, SVHN, and CIFAR-10 of our best model with the state-of-art methods in Tab.\ref{main result}.
It suggests that semi-JRFs work well in SSL, and obtain classification results comparable to the state-of-art methods.

\subsection{Generated samples}

Fig.\ref{fig:generate} shows SVHN and CIFAR-10 samples generated by well trained models of Improved-GAN and semi-JRFs.
The images in Fig.\ref{fig:generate}(a)(c) are extracted from \cite{bad-gan}.
Comparing generated samples of improved-GAN and semi-JRFs, strange patterns and collapse problem occur in the former while the latter performs realistic and diverse.

We further quantitatively compare the generation quality by inception score\cite{improved-gan}(considered consistent with human judgment) on CIFAR-10: $3.87\pm0.03$ for Improved-GAN \cite{improved-gan}, $5.08\pm0.09$ for Triple-GAN \cite{triple-gan}, $7.35\pm0.09$ for semi-JRFs.
 Triple-GAN uses a three-player adversarial learning manner without FM, and addresses the bad generation problem of GANs in SSL. The generation of semi-JRFs reaches an obviously higher inception score to Triple-GAN, which demonstrates that semi-JRFs are effective in dealing with conflict of good generation and good classification.

\section{Conclusion}

In this paper, we formally present JRFs -- 
a new family of algorithms for building deep undirected generative models, with application to semi-supervised learning.
We provide theoretical results and conduct synthetic experiments to show that JRFs work well in balancing mode missing and mode covering problems which are widely observed in GANs and VAEs, and match the empirical data distribution well.
Empirically, JRFs achieve comparable performance to the state-of-art methods in SSL on three widely adopted datasets -- MNIST, SVHN, and CIFAR-10.
In addition to good classification performance, JRFs consistently exhibit good generation which demonstrates that JRFs address the conflict between good classification and good generation troubling GAN-based methods in SSL.
Besides directed generative models, JRFs explore undirected generative models applied to SSL which perform promising and are worthwhile to be further developed.

\bibliography{JRF}
\bibliographystyle{icml2018}

\clearpage
\onecolumn
\appendix
\section{Detailed architectures}
We list the detailed architectures of semi-JRFs on MNIST, SVHN and CIFAR-10 datasets in Tab.\ref{tab:mnist}, Tab.\ref{tab:svhn} and Tab.\ref{tab:cifar} respectively.
\begin{table*}[h]
	\centering  
		\caption{\textbf{MNIST}}
	\label{tab:mnist}
	\begin{tabular}{cc}  
		\hline
		\textbf{Random Field}  &\textbf{Generator}\\  
		\hline
		Input $28\times28$ Gray Image &Noise $h$ (100dims)\\    
		\hline
		MLP 1000 units,lReLU,weight norm &MLP 500 units,sotfplus,batch norm\\
		MLP 500 units,lReLU,weight norm &MLP 500 units,sotfplus,batch norm\\
		MLP 250 units,lReLU,weight norm &MLP 784 units,sigmoid\\
		MLP 250 units,lReLU,weight norm &\\
		MLP 250 units,lReLU,weight norm &\\
		MLP 10 units,linear,weight norm &\\
		\hline
	\end{tabular}
		\vskip -0.1in
\end{table*}
\begin{table*}[h]
	\centering  
		\caption{\textbf{SVHN}}
	\label{tab:svhn}
	\begin{tabular}{cc}  
		\hline
		\textbf{Random Field}  &\textbf{Generator}\\  
		\hline
		Input $32\times32$ Colored Image &Noise $h$ (100dims)\\    
		\hline
		$3\times3$ conv. 64 lReLU,weight norm &MLP 8192 units,ReLU,batch norm\\
		$3\times3$ conv. 64 lReLU,weight norm &Reshape $512\times4\times4$\\
		$3\times3$ conv. 64 lReLU,weight norm &$5\times5$ deconv. 256 ReLU, stride=2\\
		stride=2, dropout2d=0.5 &$5\times5$ deconv. 128 ReLU, stride=2\\
		$3\times3$ conv. 128 lReLU,weight norm &$5\times5$ deconv. 3 tanh, stride=2\\
		$3\times3$ conv. 128 lReLU,weight norm &\\
		$3\times3$ conv. 128 lReLU,weight norm &\\
		stride=2, dropout2d=0.5 &\\
		$3\times3$ conv. 128 lReLU,weight norm &\\
		NIN, 128 lReLU, weight norm &\\
		NIN, 128 lReLU, weight norm &\\
		MLP 10 units, linear, weight norm &\\
		\hline
	\end{tabular}
		\vskip -0.1in
\end{table*}
\begin{table*}[h]
	\centering  
		\caption{\textbf{CIFAR-10}}
	\label{tab:cifar}
	\begin{tabular}{cc}  
		\hline
		\textbf{Random Field}  &\textbf{Generator}\\  
		\hline
		Input $32\times32$ Colored Image &Noise $h$ (100dims)\\    
		\hline
		$3\times3$ conv. 128 lReLU,weight norm &MLP 8192 units,ReLU,batch norm\\
		$3\times3$ conv. 128 lReLU,weight norm &Reshape $512\times4\times4$\\
		$3\times3$ conv. 128 lReLU,weight norm &$5\times5$ deconv. 256 ReLU, stride=2\\
		stride=2, dropout2d=0.5 &$5\times5$ deconv. 128 ReLU, stride=2\\
		$3\times3$ conv. 256 lReLU,weight norm &$5\times5$ deconv. 3 tanh, stride=2\\
		$3\times3$ conv. 256 lReLU,weight norm &\\
		$3\times3$ conv. 256 lReLU,weight norm &\\
		stride=2, dropout2d=0.5 &\\
		$3\times3$ conv. 512 lReLU,weight norm &\\
		NIN, 256 lReLU, weight norm &\\
		NIN, 128 lReLU, weight norm &\\
		MLP 10 units, linear, weight norm &\\
		\hline
	\end{tabular}
		\vskip -0.2in
\end{table*}
\section{SSL toy experiment}
Despite the proper comparison in unsupervised learning, we further exhibit the performance of semi-JRFs in SSL setting via another synthetic experiment.
The dataset is a 2D gaussian mixture with 16 mixture component uniformly distributed on two concentric circles, and the two circles represent two different classes, each class with 4 labeled data and 400 unlabeled data.
This is in fact a simple classification task that many methods can learn a proper class boundary just between the two circles, and we mainly focus on the learned data distribution $p(x),p(x|y)$ of semi-JRFs, which cannot be properly exhibited in directed generative models developed for SSL.

Fig.\ref{fig:toy2} shows learned data distribution $p_\theta(x),p_\theta(x|y=1),p_\theta(x|y=2)$ of the random field in semi-JRFs.
For each learned data distribution, semi-JRFs fit the empirical distribution well in this simple SSL setting.
\begin{figure*}[h]
	\subfigure[training set]{
		\begin{minipage}{0.24\textwidth}
			\centering  
			\includegraphics[width=\textwidth]{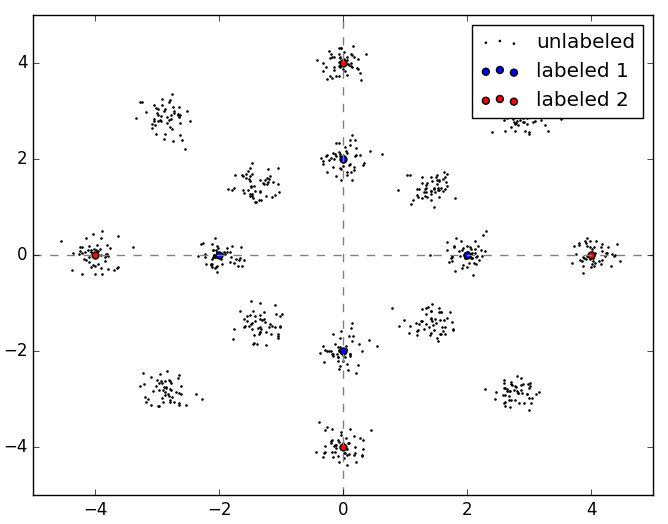}  
	\end{minipage}}
	\subfigure[learned $p_\theta(x)$]{
		\begin{minipage}{0.24\textwidth}  
			\centering  
			\includegraphics[width=\textwidth]{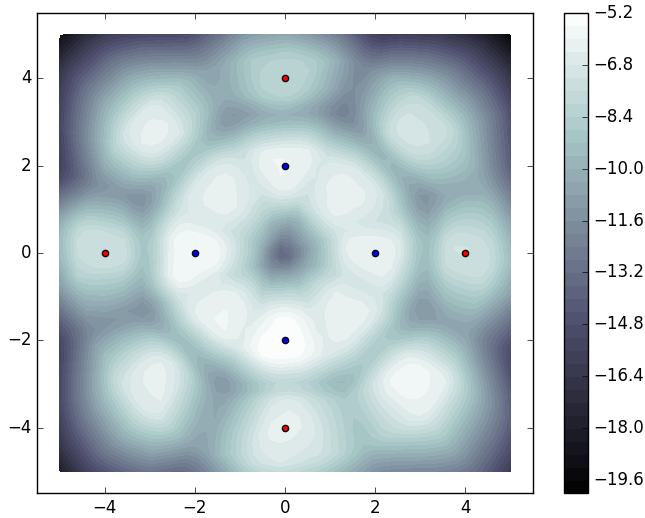}  
	\end{minipage}}
	\subfigure[learned $p_\theta(x|y=1)$]{  
		\begin{minipage}{0.24\textwidth}  
			\centering  
			\includegraphics[width=\textwidth]{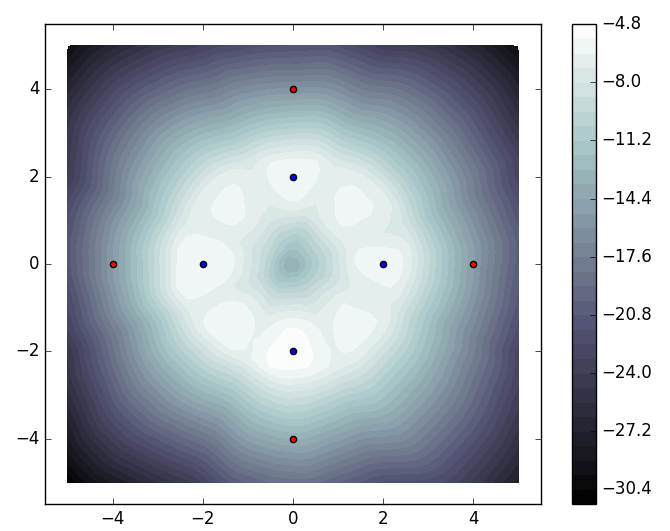}    
	\end{minipage}}
	\subfigure[learned $p_\theta(x|y=2)$]{  
		\begin{minipage}{0.24\textwidth}  
			\centering  
			\includegraphics[width=\textwidth]{fig/2_pxy1.png}    
	\end{minipage}}
	\caption{Figures of the second toy experiment. (a) data distribution of training set. 
		(b)the learned $p_\theta(x)$ of the random field of semi-JRFs.
		(c)the learned $p_\theta(x|y=1)$ of the random field.
		(d)the learned $p_\theta(x|y=2)$ of the random field.
		Each class has 4 labeled data, blue dots for class 1 and red for class 2.
		For energy density figures, white represents low energy and black for high energy.
	}
	\label{fig:toy2}
\end{figure*}  
\section{Interpolating generation}
\begin{figure}[htb]
	\centering  
	\includegraphics[width=0.8\textwidth]{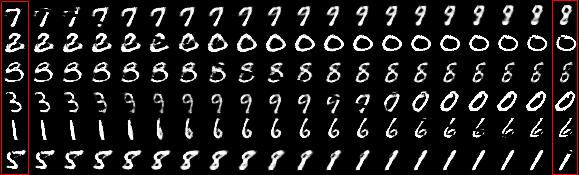}  
	\caption{Interpolation of semi-JRFs on MNIST.
		The leftmost and rightmost columns are from stochastic generation.
		The other columns show images generated by interpolating between them in latent variable space.
	} 
	\label{fig:inp}
	\vskip -0.2in
\end{figure}  

Fig.\ref{fig:inp} shows that the generator smoothly outputs transitional samples as the latent variable $h$ changed linearly, trained on MNIST dataset. 
This interpolating generation exhibits the latent variable $h$ captures some structural features of MNIST examples and is well-learned by the generator.

\section{Conditional generation}
\begin{figure}[htb]
	\centering  
	\includegraphics[width=0.6\textwidth]{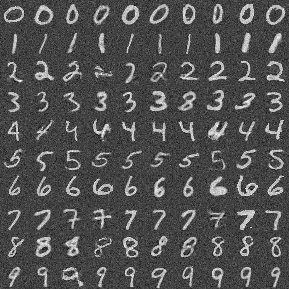}  
	\caption{A pathway of conditional generation for semi-JRFs with MNIST.
		The generating process is described in text.
		The contrast loss of images is from the revision process that the pixel value of background becomes more than 0. 
		Each line contains 10 different images conditioned on the same label.
	} 
	\label{fig:cond}
	\vskip -0.2in
\end{figure}  
Fig.\ref{fig:cond} shows a pathway of conditional generation in semi-JRFs on MNIST. 
Notice that the generator is not informed by labels, thus it cannot perform conditional generation.
While the random field can properly learn the data distribution conditioned on labels $p(x|y)$ and a revision process helps generated samples shift to low-energy region, we can execute as follows:

1) generate plenty of samples unconditionally

2) predict a label $y$ for each sample $x$ by the random field

3) a revision process starts from $x$ and towards low-energy region of $p(x|y)$

The resulting samples are regarded as conditional generation of $y$.
Considering a sample $x$ predicted with label $y$ would be a proper proposal of class $y$ examples (maybe wrongly predicted, but with similar structural features),
the revision process can effectively push $x$ to high-probability region of $p(x|y)$ where wrong prediction would be rare.
For GAN-based methods like Improved-GAN \cite{improved-gan}, the generated samples would be regarded as fake by the discriminator and thus cannot obtain a confident class-prediction.
Directly choosing samples with given prediction would accumulate the error of both nets while above-mentioned strategy for semi-JRFs would decrease the error inside the random field.

\end{document}